\newif\ifFULL
\FULLfalse	

\newif\ifnotFULL	
\notFULLtrue
\ifFULL\notFULLfalse\fi

\documentclass[10pt]{article}
\usepackage{amsmath,amssymb,amsthm}

\allowdisplaybreaks[4]

\newcommand{\Vladimir}{Vladimir}
\newcommand{\DOT}{.}

\newcommand{\e}{\mathrm{e}}
\newcommand{\bbbr}{\mathbb{R}}
\newcommand{\III}{\mathbb{I}}
\newcommand{\LLL}{\mathcal{L}}

\newcommand{\st}{\mid}

\newtheorem{theorem}{Theorem}
\newtheorem{corollary}{Corollary}
\newtheorem{lemma}{Lemma}
\theoremstyle{definition}

\newtheorem*{remark}{Remark}

\ifFULL
  \usepackage{color}

  \newcommand{\bluebegin}{\begingroup\color{blue}}
  \newcommand{\blueend}{\endgroup}

\fi

\newlength{\IndentI}
\newlength{\IndentII}
\newlength{\IndentIII}
\newlength{\IndentIV}
\newlength{\WidthI}
\newlength{\WidthII}
\newlength{\WidthIII}
\newlength{\WidthIV}
\begin{document}
\title{Prediction with expert evaluators' advice}
\author{Alexey Chernov and Vladimir Vovk\\
  \{\texttt{chernov,vovk}\}\texttt{@cs.rhul.ac.uk}
}

\maketitle

\begin{abstract}
  We introduce a new protocol for prediction with expert advice
  in which each expert evaluates the learner's and his own performance
  using a loss function that may change over time
  and may be different from the loss functions
  used by the other experts.
  The learner's goal is to perform better or not much worse
  than each expert, as evaluated by that expert,
  for all experts simultaneously.
  If the loss functions used by the experts
  are all proper scoring rules and all mixable,
  we show that the defensive forecasting algorithm
  enjoys the same performance guarantee
  as that attainable by the Aggregating Algorithm in the standard setting
  and known to be optimal.
  This result is also applied
  to the case of ``specialist'' (or ``sleeping'') experts.
  In this case,
  the defensive forecasting algorithm reduces
  to a simple modification of the Aggregating Algorithm.
\end{abstract}

\section{Introduction}

We consider the problem of online sequence prediction.
A process generates outcomes $\omega_1,\omega_2,\ldots$ step by step.
At each step $t$, a learner tries to guess the next outcome
announcing his prediction $\gamma_t$.
Then the actual outcome $\omega_t$ is revealed.
The quality of the learner's prediction is measured
by a loss function:
the learner's loss at step $t$ is $\lambda(\gamma_t,\omega_t)$.

Prediction with expert advice is a framework
that does not make any assumptions about the generating process.
The performance of the learner is compared
to the performance of several other predictors called experts.
At each step, each expert gives his prediction $\gamma^n_t$,
then the learner produces his own prediction $\gamma_t$
(possibly based on the experts' predictions at the last step
and the experts' predictions and outcomes at all the previous steps),
and the accumulated losses are updated for the learner 
and for the experts.
There are many algorithms for the learner in this framework;
for a review, see~\cite{cesabianchi/lugosi:2006}.

In practical applications of the algorithms
for prediction with expert advice,
choosing the loss function is often a problem.
The task may have no natural measure of loss,
except the vague concept that the closer the prediction
to the outcome the better.
Thus one can select among several common loss functions,
for example, the quadratic loss
(reflecting the idea of least squares methods)
or the logarithmic loss (which has an information theory background).
A similar issue arises when experts themselves are prediction algorithms
that optimize some losses internally.
Then it is unfair to these experts
when the learner competes with them according to a ``foreign'' loss function.

This paper introduces a new version of the framework of prediction with expert advice
where there is no single fixed loss function
but some loss function is linked to every expert.
The performance of the learner is compared to
the performance of each expert according to the loss
function linked to that expert.
Informally speaking,
each expert has to be convinced
that the learner performs almost as well as, or better than,
that expert himself.

We prove that a known algorithm for the learner,
the defensive forecasting algorithm
\cite{chernov/etal:2008supermartingales},
can be applied in the new setting
and gives the same performance guarantee as that attainable in the standard setting,
\ifnotFULL
  provided all loss functions are proper scoring rules.
\fi

Another framework to which our methods can be fruitfully applied
is that of ``specialist experts'':
see, e.g., \cite{freund/etal:1997}, \cite{blum/mansour:2007},
and \cite{kleinberg/etal:2008}.
We generalize some of the known results
in the case of mixable loss functions.

To keep presentation as simple as possible,
we restrict ourselves to binary outcomes $\{0,1\}$,
predictions from $[0,1]$, and a finite number of experts.
We formulate our results for mixable loss functions only.
However, these results can be easily transferred 
to more general settings (non-binary outcomes,
arbitrary prediction spaces, countably many experts, second-guessing experts, etc.)\
where the methods of~\cite{chernov/etal:2008supermartingales} work.

\section{Prediction with simple experts' advice}
\label{sec:PEA}

In this preliminary section
we recall the standard protocol of prediction with expert advice
and some known results.

Let $\{0,1\}$ be the set of possible \emph{outcomes}~$\omega$,
$[0,1]$ be the set of possible \emph{predictions}~$\gamma$,
and ${\lambda:[0,1]\times\{0,1\}\to[0,\infty]}$
be the \emph{loss function}. 
The loss function $\lambda$ and parameter $N$
(the number of experts)
specify the game of prediction with expert advice.
The game is played by Learner, Reality, and $N$ experts,
Expert 1 to Expert $N$,
according to the following protocol.

\bigskip

\noindent
\mbox{\textsc{Prediction with expert advice}}\nopagebreak

\smallskip

\parshape=9
\IndentI  \WidthI
\IndentI  \WidthI
\IndentI  \WidthI
\IndentII \WidthII
\IndentII \WidthII
\IndentII \WidthII
\IndentII \WidthII
\IndentII \WidthII
\IndentI  \WidthI
\noindent
$L_0:=0$.\\ 
$L_0^n:=0$, $n=1,\ldots,N$.\\
FOR $t=1,2,\dots$:\\
  Expert $n$ announces $\gamma_t^n\in[0,1]$, $n=1,\ldots,N$.\\
  Learner announces $\gamma_t\in[0,1]$.\\
  Reality announces $\omega_t\in\{0,1\}$.\\
  $L_t:=L_{t-1}+\lambda(\gamma_t,\omega_t)$.\\
  $L_t^n:=L_{t-1}^n+\lambda(\gamma_t^n,\omega_t)$,
    $n=1,\ldots,N$.\\
END FOR

\bigskip

\noindent
The goal of Learner is to keep his loss $L_t$
smaller or at least not much greater than the loss $L_t^n$ of Expert~$n$,
at each step $t$ and for all $n=1,\ldots,N$.

We only consider loss functions that have the following properties:
\begin{description}
\item[Assumption 1:]
  $\lambda(\gamma,0)$ and $\lambda(\gamma,1)$
  are continuous in $\gamma\in[0,1]$
  and for the standard (Aleksandrov's) topology on $[0,\infty]$.
\item[Assumption 2:]
  There exists $\gamma\in[0,1]$
  such that $\lambda(\gamma,0)$
  and $\lambda(\gamma,1)$ are both finite.
\item[Assumption 3:]
  There exists no $\gamma\in[0,1]$
  such that $\lambda(\gamma,0)$
  and $\lambda(\gamma,1)$ are both infinite.
\end{description}

The \emph{superprediction set} for a loss function $\lambda$ is
\begin{equation}\label{eq:Sigma}
  \Sigma_{\lambda}
  :=
  \left\{
    (x,y)\in[0,\infty)^2
    \st
    \exists\gamma\,\lambda(\gamma,0)\le x\text{ and }\lambda(\gamma,1)\le y
  \right\}.
\end{equation}
By Assumption 2, this set is non-empty.
For $\eta>0$,
let $E_{\eta}:[0,\infty]^2\to[0,1]^2$
be the homeomorphism defined by
$E_{\eta}(x,y):=(\e^{-\eta x},\e^{-\eta y})$.
The loss function $\lambda$ is called \emph{$\eta$-mixable}
if the set $E_{\eta}(\Sigma_{\lambda})$ is convex.
It is called \emph{mixable} if it is $\eta$-mixable for some $\eta>0$.

\begin{theorem}\label{thm:standard}
  If a loss function $\lambda$ is $\eta$-mixable,
  then there exists a strategy for Learner
  that guarantees that in the game of prediction with expert advice
  with $N$ experts and the loss function $\lambda$
  it holds,
  for all $t$
  and for all $n=1,\ldots,N$,
  that
  \begin{equation}\label{eq:standard}
    L_t\le L_t^n + \frac{1}{\eta} \ln N.
  \end{equation}
  The bound is optimal:
  if $\lambda$ is not $\eta$-mixable,
  then no strategy for Learner
  can guarantee (\ref{eq:standard}).
\end{theorem}

For the proof and other details, 
see \cite{cesabianchi/lugosi:2006}, \cite{haussler/etal:1998},
\cite{vovk:1998game}, or \cite[Theorem~8]{vovk:1999derandomizing};
one of the algorithms guaranteeing (\ref{eq:standard})
is the (Strong) Aggregating Algorithm (AA).
As shown in~\cite{chernov/etal:2008supermartingales},
one can take the defensive forecasting algorithm
instead of the AA in the theorem.

\section{Proper scoring rules}
\label{sec:PSR}

A loss function $\lambda$ is a \emph{proper scoring rule}
if for any $\pi,\pi'\in[0,1]$ it holds that
$$
  \pi\lambda(\pi,1)+(1-\pi)\lambda(\pi,0)
  \le
  \pi\lambda(\pi',1)+(1-\pi)\lambda(\pi',0);
$$
it is a \emph{strictly proper scoring rule}
if the inequality holds with $<$ in place of $\le$ whenever $\pi'\ne\pi$.
The interpretation is that the prediction $\pi$
is an estimate of the probability that $\omega=1$.
The definition says that the expected loss
with respect to a probability distribution is minimal
if the prediction is the true probability of $1$.
Informally, a strictly proper scoring rule
encourages a forecaster (Learner or one of the experts)
to announce his true subjective probability that the next outcome is $1$.
(See \cite{dawid:1986}, \cite{gneiting/raftery:2007},
and \cite{buja/etal:2005} for detailed reviews.)

Simple examples of strictly proper scoring rules
are provided by two most common loss functions:
the log loss function
$$
  \lambda(\gamma,\omega)
  :=
  -\ln(\omega\gamma+(1-\omega)(1-\gamma))
$$
(i.e., $\lambda(\gamma,0)=-\ln(1-\gamma)$ and $\lambda(\gamma,1)=-\ln\gamma$)
and the square loss function
$$
  \lambda(\gamma,\omega)
  :=
  (\omega-\gamma)^2\,.
$$
A trivial but important for us generalization of the log loss function is
\begin{equation}\label{eq:generalized-log}
  \lambda(\gamma,\omega)
  :=
  -\frac{1}{\eta}\ln(\omega\gamma+(1-\omega)(1-\gamma)),
\end{equation}
where $\eta$ is a positive constant.
The generalized log loss function is also a proper scoring rule
(in general, multiplying a proper scoring rule by a positive constant
we again obtain a proper scoring rule).

We will often say ``(strictly) proper loss function''
meaning a loss function that is a (strictly) proper scoring rule.
Our main interest will be in loss functions that are both mixable and proper.
Let $\LLL$ be the set of all such loss functions.


\section{Prediction with expert evaluators' advice}
\label{sec:PEAA}

In this section we consider a very general protocol
of prediction with expert advice.
The intuition behind special cases of this protocol will be discussed
in the following sections.

\bigskip

\noindent
\mbox{\textsc{Prediction with expert evaluators' advice}}\nopagebreak

\smallskip

\parshape=6
\IndentI   \WidthI
\IndentII  \WidthII
\IndentIII \WidthIII
\IndentII  \WidthII
\IndentII  \WidthII
\IndentI   \WidthI
\noindent
FOR $t=1,2,\dots$:\\
  Expert $n$ announces $\gamma_t^n\in[0,1]$, $\eta_t^n>0$,
    and $\eta_t^n$-mixable $\lambda_t^n\in\LLL$,\\
    $n=1,\ldots,N$.\\
  Learner announces $\gamma_t\in[0,1]$.\\
  Reality announces $\omega_t\in\{0,1\}$.\\
END FOR

\bigskip

The main mathematical result of this paper is the following.
\begin{theorem}\label{thm:main}
  Learner has a strategy
  (e.g., the defensive forecasting algorithm described below)
  that guarantees that in the game of prediction
  with $N$ expert evaluators' advice
  it holds,
  for all $T$
  and for all $n=1,\ldots,N$,
  that
  $$
    \sum_{t=1}^T
    \eta^n_t
    \bigl(
      \lambda^n_t(\pi_t,\omega_t)
      -
      \lambda^n_t(\gamma^n_t,\omega_t)
    \bigr)
    \le
    \ln N.
  $$
\end{theorem}
\noindent
The description of the defensive forecasting algorithm
and the proof of the theorem
will be given in Section \ref{sec:proof}.
\begin{corollary}\label{cor:main}
  For any $\eta>0$,
  Learner has a strategy that guarantees
  \begin{equation}\label{eq:guarantee}
    \sum_{t=1}^T
    \lambda^n_t(\pi_t,\omega_t)
    \le
    \sum_{t=1}^T
    \lambda^n_t(\gamma^n_t,\omega_t)
    +
    \frac{\ln N}{\eta},
  \end{equation}
  for all $T$ and all $n=1,\ldots,N$,
  in the game of prediction with $N$ expert evaluators' advice
  in which the experts are required to always choose
  $\eta$-mixable loss functions $\lambda^n_t$.
\end{corollary}
\noindent
This corollary is more intuitive than Theorem \ref{thm:main}
as (\ref{eq:guarantee}) compares the cumulative losses suffered by Learner
and each expert.

In the following sections we will discuss
two interesting special cases
of Theorem \ref{thm:main} and Corollary \ref{cor:main}.

\section{Prediction with constant expert evaluators' advice}
\label{sec:PCEAA}

In the game of this section,
as in the previous one,
the experts are ``expert evaluators'':
each of them measures Learner's and his own performance
using his own loss function,
supposed to be mixable and proper.
The difference is that now each expert is linked
to a fixed loss function.
The game is specified by $N$ loss functions
$\lambda^1,\ldots,\lambda^N$.

\bigskip

\noindent
\mbox{\textsc{Prediction with constant expert evaluators' advice}}\nopagebreak

\smallskip

\parshape=9
\IndentI  \WidthI
\IndentI  \WidthI
\IndentI  \WidthI
\IndentII \WidthII
\IndentII \WidthII
\IndentII \WidthII
\IndentII \WidthII
\IndentII \WidthII
\IndentI  \WidthI
\noindent
$L_0^{(n)}:=0$, $n=1,\ldots,N$.\\
$L_0^n:=0$, $n=1,\ldots,N$.\\
FOR $t=1,2,\dots$:\\
  Expert $n$ announces $\gamma_t^n\in[0,1]$, $n=1,\ldots,N$.\\
  Learner announces $\gamma_t\in[0,1]$.\\
  Reality announces $\omega_t\in\{0,1\}$.\\
  $L_t^{(n)}:=L_{t-1}^{(n)}+\lambda^n(\gamma_t,\omega_t)$,
    $n=1,\ldots,N$.\\
  $L_t^n:=L_{t-1}^n+\lambda^n(\gamma_t^n,\omega_t)$,
    $n=1,\ldots,N$.\\
END FOR

\bigskip

There are two changes in the protocol
as compared to the basic protocol of prediction with expert advice
in Section \ref{sec:PEA}.
The accumulated loss $L_t^n$ of each expert is now calculated
according to his own loss function~$\lambda^n$.
For Learner,
there is no single accumulated loss anymore.
Instead, the loss $L_t^{(n)}$ of Learner
is calculated  separately against each expert, 
according to that expert's loss function~$\lambda^n$.
Informally speaking, each expert evaluates his own performance
and the performance of Learner
according to the expert's own (but publicly known) criteria.

In the standard setting of prediction with expert advice
it is often said that Learner's goal 
is to compete with the best expert in the pool.
In the new setting, 
we cannot speak about the best expert: 
the experts' performance is evaluated by different loss functions 
and thus the losses may be measured on different scales.
But it still makes sense to consider bounds
on the \emph{regret} $L_t^{(n)}-L_t^n$ for each~$n$.

Theorem \ref{thm:main} (or Corollary \ref{cor:main})
immediately implies the following performance guarantee
for the defensive forecasting algorithm in our current setting.
\begin{corollary}\label{cor:multbound}
  Suppose that every $\lambda^n$
  is a proper loss function that is $\eta^n$-mixable for some $\eta^n>0$,
  $n=1,\ldots,N$.
  Then Learner has a strategy
  (such as the defensive forecasting algorithm)
  that guarantees that in the game of prediction
  with $N$ experts' advice and loss functions $\lambda^1,\ldots,\lambda^N$
  it holds,
  for all $T$
  and for all $n=1,\ldots,N$,
  that
  \begin{equation}\label{eq:multbound}
    L^{(n)}_T
    \le
    L_T^n + \frac{\ln N}{\eta^n}.
  \end{equation}
\end{corollary}
\noindent
The new bound (\ref{eq:multbound}) is precisely the same
as the bound for the standard setting of Theorem~\ref{thm:standard}.
But rigorous comparison of the actual power of these two bounds
is not so trivial.

Formally speaking, the task of Learner in the new protocol
is not strictly harder and is not strictly easier than in the standard protocol:
the task is incomparable.
Learner must now compete with different experts by different rules.
But this is not necessarily a disadvantage.
Consider an example.
Suppose that all experts except one 
are linked to one loss function and 
the last expert is linked to another loss function.
And this last loss function is somehow trivial, 
say, equals $1$ independent of the outcome and the prediction.
Then we arrive at the standard protocol with $N-1$ experts,
since the regret against the last expert 
is zero independent of our predictions.
In this example, we can get a better bound
than that given by Corollary \ref{cor:multbound}.
This non-optimality is especially apparent in the case
when we have a huge number of experts,
but all except one are linked to a trivial loss function.
Then our regret bound is large,
being a logarithm of a huge number, 
whereas one can achieve zero regret against all experts
whatever strategy they use---%
since the loss functions are unfavourable to the experts.

Nevertheless, it is intuitively clear that the new protocol
is somewhat harder for Learner in general.
And Corollary~\ref{cor:multbound} is really surprising:
it is hard to believe that Learner can compete
against several arbitrary loss functions
as well as against only one of them.
The reason why this is possible is that the loss functions
are assumed to be proper.

\subsection*{Multiobjective prediction with expert advice}

To conclude this section,
let us consider another variant of the protocol 
with several loss functions.
As mentioned in the introduction,
sometimes we have experts' predictions,
and we are not given a single loss function,
but have several possible candidates.
The most cautious way to generate Learner's predictions
is to ensure that the regret is small
against all experts and according to all loss functions.
The following protocol formalizes this task.
Now we have $N$ experts and $M$ loss functions 
$\lambda^1,\ldots,\lambda^M$.

\bigskip

\noindent
\mbox{\textsc{Multiobjective prediction with expert advice}}\nopagebreak

\smallskip

\parshape=9
\IndentI  \WidthI
\IndentI  \WidthI
\IndentI  \WidthI
\IndentII \WidthII
\IndentII \WidthII
\IndentII \WidthII
\IndentII \WidthII
\IndentII \WidthII
\IndentI  \WidthI
\noindent
$L_0^{(m)}:=0$, $m=1,\ldots,M$.\\
$L_0^{n,m}:=0$, $n=1,\ldots,N$ and $m=1,\ldots,M$.\\
FOR $t=1,2,\dots$:\\
  Expert $n$ announces $\gamma_t^n\in[0,1]$, $n=1,\ldots,N$.\\
  Learner announces $\gamma_t\in[0,1]$.\\
  Reality announces $\omega_t\in\{0,1\}$.\\
  $L_t^{(m)}:=L_{t-1}^{(m)}+\lambda^m(\gamma_t,\omega_t)$,
    $m=1,\ldots,M$.\\
  $L_t^{n,m}:=L_{t-1}^{n,m}+\lambda^m(\gamma_t^n,\omega_t)$,
    $n=1,\ldots,N$ and $m=1,\ldots,M$.\\
END FOR

\bigskip

\begin{corollary}\label{cor:multiobjective}
  Suppose that every $\lambda^m$ is an $\eta^m$-mixable proper loss function,
  for some $\eta^m>0$, $m=1,\ldots,M$.
  The defensive forecasting algorithm guarantees that,
  in the multiobjective game of prediction
  with $N$ experts and the loss functions $\lambda^1,\ldots,\lambda^M$,
  \begin{equation}\label{eq:multiobjective}
    L^{(m)}_t\le L_t^{n,m} + \frac{\ln MN}{\eta^m}
  \end{equation}
  for all $t$, all $n=1,\ldots,N$, and all $m=1,\ldots,M$.
\end{corollary}
\begin{proof}
  This follows easily from Corollary \ref{cor:multbound}.
  For each $n\in\{1,\ldots,N\}$, let us construct $M$ new experts $(n,m)$.
  Expert $(n,m)$ predicts as Expert $n$
  and is linked to the loss function $\lambda^m$.
  Applying Corollary \ref{cor:multbound} to these $MN$ experts,
  we get bound (\ref{eq:multiobjective}).
\end{proof}

The last protocol is harder for Learner
than the standard protocol when $M>1$:
Learner must satisfy all old regret bounds and also some new bounds.
But the increase in the regret bounds is surprisingly small:
only an additive term proportional to $\ln M$.
Whether the dependence on $M$ in Corollary \ref{cor:multiobjective} is optimal
remains an open problem.

A further generalization of our last protocol
involves a binary relation $R$
between the $N$ experts and the $M$ loss functions,
where $nRm$, $n\in\{1,\ldots,N\}$ and $m\in\{1,\ldots,M\}$,
is interpreted as Expert $n$
using the loss function $\lambda^m$
when evaluating Learner's and his own performance.
It is assumed that for each $n$ there exists at least one $m$
such that $nRm$.
The relation $R$ is naturally represented as a bipartite graph
connecting the vertices in the set $\{1,\ldots,N\}$
to vertices in the set $\{1,\ldots,M\}$.
Equation (\ref{eq:multiobjective}) now becomes
\begin{equation*}
  L^{(m)}_t\le L_t^{n,m} + \frac{\ln K}{\eta^m},
\end{equation*}
for all $(n,m)\in R$,
where $K$ is the cardinality of $R$
(equivalently, the number of edges in the bipartite graph).

\subsection*{A simple example}

Let $\lambda^1$ be the log loss function
and $\lambda^2$ the square loss function.
As already mentioned,
both loss functions are proper and mixable.
It is known
(see, e.g., \cite{cesabianchi/lugosi:2006},
\cite{haussler/etal:1998}, or \cite{vovk:1990})
that $\lambda^1$ is $1$-mixable and $\lambda^2$ is $2$-mixable.
%
Suppose we are competing with $N$ experts
producing predictions $\gamma^n_t$
under these two loss functions.
The defensive forecasting algorithm ensures that the regret
with respect to the logarithmic loss function
is bounded by $\ln(2N)<\ln N + 0.7$,
and the regret with respect to the square loss function
is bounded by $0.5\ln(2N)<0.5\ln N + 0.4$---%
practically the same as the regrets against $N$ experts
that are achievable when Learner chooses his predictions
with respect to one of the loss functions only.

\section{Prediction with specialist experts' advice}
\label{sec:sleeping}

The experts of this section are allowed to ``sleep'',
i.e.,
abstain from giving advice to Learner at some steps.
This generalization is important for text-processing applications
(see, e.g., \cite{cohen/singer:1999}).
We will be assuming that there is only one loss function $\lambda$,
although generalization to the case of $N$ loss functions
$\lambda^1,\ldots,\lambda^N$
is straightforward.
The loss function $\lambda$ does not need to be proper
(but it is still required to be mixable).

Let $a$ be any object that does not belong to $[0,1]$;
intuitively, it will stand for an expert's decision to abstain.

\bigskip

\noindent
\mbox{\textsc{Prediction with specialist experts' advice}}\nopagebreak

\smallskip

\parshape=9
\IndentI  \WidthI
\IndentI  \WidthI
\IndentI  \WidthI
\IndentII \WidthII
\IndentII \WidthII
\IndentII \WidthII
\IndentII \WidthII
\IndentII \WidthII
\IndentI  \WidthI
\noindent
$L_0^{(n)}:=0$, $n=1,\ldots,N$.\\
$L_0^n:=0$, $n=1,\ldots,N$.\\
FOR $t=1,2,\dots$:\\
  Expert $n$ announces $\gamma_t^n\in([0,1]\cup\{a\})$, $n=1,\ldots,N$.\\
  Learner announces $\gamma_t\in[0,1]$.\\
  Reality announces $\omega_t\in\{0,1\}$.\\
  $L_t^{(n)}:=L_{t-1}^{(n)}+\III_{\{\gamma_t^n\ne a\}}\lambda(\gamma_t,\omega_t)$,
    $n=1,\ldots,N$.\\
  $L_t^n:=L_{t-1}^n+\III_{\{\gamma_t^n\ne a\}}\lambda(\gamma_t^n,\omega_t)$,
    $n=1,\ldots,N$.\\
END FOR

\bigskip

\noindent
The indicator function $\III_{\{\gamma_t^n\ne a\}}$ of the event $\gamma_t^n\ne a$
is defined to be 1 if $\gamma_t^n\ne a$ and 0 if $\gamma_t^n=a$.
Therefore, $L_t^{(n)}$ and $L_t^{n}$
refer to the cumulative loss of Learner and Expert $n$
over the steps when Expert $n$ is awake.
Now Learner's goal is to do as well as each expert
on the steps chosen by that expert.

\begin{corollary}\label{cor:sleeping}
  Let $\lambda$ be a loss function
  that is $\eta$-mixable for some $\eta>0$.
  Then Learner has a strategy
  (e.g., the defensive forecasting algorithm)
  that guarantees that in the game of prediction
  with $N$ specialist experts' advice and loss function $\lambda$
  it holds,
  for all $T$
  and for all $n=1,\ldots,N$,
  that
  \begin{equation}\label{eq:sleeping}
    L^{(n)}_T
    \le
    L_T^n + \frac{\ln N}{\eta}.
  \end{equation}
\end{corollary}
\begin{proof}
  Without loss of generality the loss function $\lambda$
  may be assumed to be proper
  (this can be achieved by reparameterization of the predictions
  $\gamma\in[0,1]$).
  The protocol of this section
  then becomes a special case of the protocol of Section \ref{sec:PEAA}
  in which at each step each expert outputs $\eta_t^n=\eta$
  and either $\lambda_t^n=\lambda$ (when he is awake)
  or $\lambda_t^n=0$ (when he is asleep).
  (Alternatively,
  in which at each step each expert outputs $\lambda_t^n=\lambda$
  and either $\eta_t^n=\eta$, when he is awake,
  or $\eta_t^n=0$, when he is asleep.)

\end{proof}

\section{Defensive forecasting algorithm
  and the proof of Theorem \ref{thm:main}}
\label{sec:proof}

In this section we prove Theorem~\ref{thm:main}.
Our proof is constructive:
we explicitly describe the defensive forecasting algorithm
achieving the bound in Theorem~\ref{thm:main}.

\subsection*{The algorithm}

For each $n=1,\ldots,N$, 
let us define the function
\begin{align}
  &Q^n:
  \left(
    [0,1]^N \times (0,\infty)^N \times \LLL^N \times [0,1] \times \{0,1\}
  \right)^*
  \to
  [0,\infty]
  \notag\\
  &Q^n
  \left(
    \gamma_1^{\bullet},\eta_1^{\bullet},\lambda_1^{\bullet},\pi_1,\omega_1,\ldots,
    \gamma_T^{\bullet},\eta_T^{\bullet},\lambda_T^{\bullet},\pi_T,\omega_T
  \right)
  :=
  \prod_{t=1}^T
  \e^{
    \eta^n_t \bigl(
      \lambda^n_t(\pi_t,\omega_t)
      -
      \lambda^n_t(\gamma^n_t,\omega_t)
    \bigr)
  },
  \label{eq:Q-n}
\end{align}
where $\gamma^n_t$ are the components of $\gamma^{\bullet}_t$,
$\eta^n_t$ are the components of $\eta^{\bullet}_t$,
and $\lambda^n_t$ are the components of $\lambda^{\bullet}_t$:
\begin{align*}
  \gamma^{\bullet}_t &:= (\gamma_t^1,\ldots,\gamma_t^N),\\
  \eta^{\bullet}_t &:= (\eta_t^1,\ldots,\eta_t^N),\\
  \lambda^{\bullet}_t &:= (\lambda_t^1,\ldots,\lambda_t^N).
\end{align*}
As usual, the product $\prod_{t=1}^0$ is interpreted as $1$,
so that $Q^n()=1$.
The functions $Q^n$ will usually be applied to
$\gamma^{\bullet}_t:=(\gamma_t^1,\ldots,\gamma_t^N)$
the predictions made by all the $N$ experts at step $t$,
$\eta^{\bullet}_t:=(\eta_t^1,\ldots,\eta_t^N)$
the learning rates chosen by the experts at step $t$,
and $\lambda^{\bullet}_t:=(\lambda_t^1,\ldots,\lambda_t^N)$
the loss functions used by the experts at step $t$.
Notice that $Q^n$ does not depend on the predictions,
learning rates, and loss functions
of the experts other than Expert $n$.

Set
$$
  Q
  :=
  \frac{1}{N} \sum_{n=1}^N Q^n
$$
and
\begin{multline}\label{eq:f}
  f_t(\pi,\omega)
  :=\\
  Q
  \left(
    \gamma^{\bullet}_1,\eta^{\bullet}_1,\lambda^{\bullet}_1,
    \pi_1,\omega_1,\ldots,
    \gamma^{\bullet}_{t-1},\eta^{\bullet}_{t-1},\lambda^{\bullet}_{t-1},
    \pi_{t-1},\omega_{t-1},
    \gamma^{\bullet}_t,\eta^{\bullet}_t,\lambda^{\bullet}_t,
    \pi,\omega
  \right)\\
  -
  Q
  \left(
    \gamma^{\bullet}_1,\eta^{\bullet}_1,\lambda^{\bullet}_1,
    \pi_1,\omega_1,\ldots,
    \gamma^{\bullet}_{t-1},\eta^{\bullet}_{t-1},\lambda^{\bullet}_{t-1},
    \pi_{t-1},\omega_{t-1}
  \right),
\end{multline}
where $(\pi,\omega)$ ranges over $[0,1]\times\{0,1\}$;
the expression $\infty-\infty$ is understood as, say, $0$.
The defensive forecasting algorithm is defined
in terms of the functions $f_t$.

\bigskip

\noindent
\mbox{\textsc{Defensive forecasting algorithm}}\nopagebreak

\smallskip

\parshape=11
\IndentI   \WidthI
\IndentII  \WidthII
\IndentIII \WidthIII
\IndentIII \WidthIII
\IndentII  \WidthII
\IndentII  \WidthII
\IndentII  \WidthII
\IndentII  \WidthII
\IndentIII \WidthIII
\IndentII  \WidthII
\IndentI   \WidthI
\noindent
FOR $t=1,2,\dots$:\\
  Read the experts' predictions
      $\gamma^{\bullet}_t=(\gamma_t^1,\ldots,\gamma_t^N)\in[0,1]^N$,\\
    learning rates $\eta^{\bullet}_t=(\eta_t^1,\ldots,\eta_t^N)\in(0,\infty)^N$,\\
    and loss functions
      $\lambda^{\bullet}_t=(\lambda_t^1,\ldots,\lambda_t^N)\in\LLL^N$.\\
  Define $f_t:[0,1]\times\{0,1\}\to[-\infty,\infty]$ by (\ref{eq:f}).\\
  If $f_t(0,1)\le0$, predict $\pi_t:=0$ and go to R.\\
  If $f_t(1,0)\le0$, predict $\pi_t:=1$ and go to R.\\
  Otherwise
  (if both $f_t(0,1)>0$ and $f_t(1,0)>0$),\\
    take any $\pi$ satisfying $f_t(\pi,0)=f_t(\pi,1)$ and predict $\pi_t:=\pi$.\\
  R: Read Reality's move $\omega_t\in\{0,1\}$.\\
END FOR

\bigskip

\noindent
The existence of a $\pi$ satisfying $f_t(\pi,0)=f_t(\pi,1)$
will be proved in Lemma \ref{lem:levin} below.
We will see that the function $f_t(\pi):=f_t(\pi,1)-f_t(\pi,0)$
takes values of opposite signs at $\pi=0$ and $\pi=1$.
Therefore, a root of $f_t(\pi)=0$ can be found by, e.g., bisection
(see \cite{press/etal:1992}, Chapter 9,
for a review of bisection and more efficient methods,
such as Brent's).

\subsection*{Reductions}

The most important property of the defensive forecasting algorithm
is that it produces predictions $\pi_t$
such that the sequence
\begin{equation}\label{eq:Q}
  Q_t
  :=
  Q(\gamma^{\bullet}_1,\eta^{\bullet}_1,\lambda^{\bullet}_1,\pi_1,\omega_1,\ldots,
    \gamma^{\bullet}_t,\eta^{\bullet}_t,\lambda^{\bullet}_t,\pi_t,\omega_t)
\end{equation}
is non-increasing.
This property will be proved later;
for now, we will only check that it implies
the bound on the regret term given in Theorem \ref{thm:main}.
Since the initial value $Q_0$ of $Q$ is $1$, we have $Q_t\le1$ for all $t$.
And since $Q^n\ge 0$ for all $n$,
we have $Q^n\le N Q$ for all $n$.
Therefore, $Q^n_t$,
defined by (\ref{eq:Q}) with $Q^n$ in place of $Q$,
is at most $N$ at each step $t$.
By the definition of $Q^n$ this means that
$$
  \sum_{t=1}^T
  \eta^n_t
  \bigl(
    \lambda^n_t(\pi_t,\omega_t)
    -
    \lambda^n_t(\gamma^n_t,\omega_t)
  \bigr)
  \le
  \ln N,
$$
which is the bound claimed in the theorem.

In the proof of the inequalities $Q_0\ge Q_1\ge\cdots$
we will follow \cite{chernov/etal:2008supermartingales}
(for a presentation adapted to the binary case,
see \cite{vovk:arXiv0708.1503}).
The key fact we use is that $Q$ is a game-theoretic supermartingale.
Let us define this notion and prove its basic properties.

Let $E$ be any non-empty set.
A function $S:(E\times[0,1]\times\{0,1\})^*\to(-\infty,\infty]$
is called a \emph{supermartingale}
(omitting ``game-theoretic'')
if, for any $T$,
any $e_1,\ldots,e_T\in E$,
any $\pi_1,\ldots,\pi_T\in[0,1]$,
and any $\omega_1,\ldots,\omega_{T-1}\in\{0,1\}$,
it holds that
\begin{multline}\label{eq:super}
  \pi_T
  S(e_1,\pi_1,\omega_1,\ldots,e_{T-1},\pi_{T-1},\omega_{T-1},e_T,\pi_T,1)\\
  +
  (1-\pi_T)
  S(e_1,\pi_1,\omega_1,\ldots,e_{T-1},\pi_{T-1},\omega_{T-1},e_T,\pi_T,0)\\
  \le
  S(e_1,\pi_1,\omega_1,\ldots,e_{T-1},\pi_{T-1},\omega_{T-1}).
\end{multline}

\begin{remark}
  The standard measure-theoretic notion of a supermartingale
  is obtained when the arguments $\pi_1,\pi_2,\ldots$ in (\ref{eq:super})
  are replaced by the forecasts produced by a fixed forecasting system.
  See, e.g., \cite{shafer/vovk:2001} for details.
  Game-theoretic supermartingales are referred to
  as ``superfarthingales'' in \cite{dawid/vovk:1999}.
\end{remark}

A supermartingale $S$ is called \emph{forecast-continuous} if,
for all $T\in\{1,2,\ldots\}$,
all $e_1,\ldots,e_T\in E$,
all $\pi_1,\ldots,\pi_{T-1}\in[0,1]$,
and all $\omega_1,\ldots,\omega_T\in\{0,1\}$,
$$
  S(e_1,\pi_1,\omega_1,\ldots,e_{T-1},\pi_{T-1},\omega_{T-1},e_T,\pi,\omega_T)
$$
is a continuous function of $\pi\in[0,1]$.
The following lemma states the most important for us property
of forecast-continuous supermartingales.

\begin{lemma}\label{lem:levin}
  Let $S$ be a forecast-continuous supermartingale.
  For any $T$ and for any values of the arguments
  $e_1,\ldots,e_T\in E$,
  $\pi_1,\ldots,\pi_{T-1}\in[0,1]$,
  and $\omega_1,\ldots,\omega_{T-1}\in\{0,1\}$,
  there exists $\pi\in[0,1]$ such that,
  for both $\omega=0$ and $\omega=1$,
  \begin{multline*}
    S(e_1,\pi_1,\omega_1,\ldots,e_{T-1},\pi_{T-1},\omega_{T-1},e_T,\pi,\omega)\\
    \le
    S(e_1,\pi_1,\omega_1,\ldots,e_{T-1},\pi_{T-1},\omega_{T-1})\,.
  \end{multline*}
\end{lemma}
\begin{proof}
  Define a function $f:[0,1]\times\{0,1\}\to(-\infty,\infty]$ by
  \begin{multline*}
    f(\pi,\omega)
    :=
    S(e_1,\pi_1,\omega_1,\ldots,e_{T-1},\pi_{T-1},\omega_{T-1},e_T,\pi,\omega)\\
    -
    S(e_1,\pi_1,\omega_1,\ldots,e_{T-1},\pi_{T-1},\omega_{T-1})
  \end{multline*}
  (the subtrahend is assumed finite:
  there is nothing to prove when it is infinite).
  Since $S$ is a forecast-continuous supermartingale, 
  $f(\pi,\omega)$ is continuous in $\pi$ and
  \begin{equation}\label{eq:mean}
    \pi f(\pi,1)+(1-\pi)f(\pi,0)
    \le
    0
  \end{equation}
  for all $\pi\in[0,1]$.
  In particular,
  $f(0,0)\le 0$ and $f(1,1)\le 0$.

  Our goal is to show that for some $\pi\in[0,1]$
  we have $f(\pi,1)\le 0$ and $f(\pi,0)\le 0$.
  If $f(0,1)\le 0$, we can take $\pi=0$.
  If $f(1,0)\le 0$, we can take $\pi=1$.
  Assume that $f(0,1)>0$ and $f(1,0)>0$.
  Then the difference
  $$
    f(\pi)
    :=
    f(\pi,1)
    -
    f(\pi,0)
  $$
  is positive for $\pi=0$ and negative for $\pi=1$.
  By the intermediate value theorem,
  $f(\pi)=0$ for some $\pi\in(0,1)$.
  By (\ref{eq:mean})
  we have $f(\pi,1)=f(\pi,0)\le 0$.
\end{proof}

The fact that the sequence (\ref{eq:Q}) is non-increasing
follows from the fact (see below) that $Q$ is a supermartingale
(when restricted to the allowed moves for the players).
The proof of Lemma \ref{lem:levin},
as applied to the supermartingale $Q$,
is summarized in (\ref{eq:f}),
the pseudocode for the defensive forecasting algorithm,
and the paragraph following it.

The weighted sum of finitely many forecast-continuous supermartingales
taken with positive weights
is again a forecast-continuous supermartingale.
Therefore,
the proof will be complete
if we check that $Q^n$ is a forecast-continuous supermartingale
under the restriction that $\lambda^n_t$ is $\eta^n_t$-mixable
for all $n$ and $t$.
But before we can do this,
we will need to do some preparatory work in the next subsection.

\subsection*{Geometry of mixability and proper loss functions}

Assumption 1 and the compactness of $[0,1]$ imply that
the superprediction set (\ref{eq:Sigma}) is closed.
Along with the superprediction set,
we will also consider the \emph{prediction set}
\begin{equation*}
  \Pi_{\lambda}
  :=
  \left\{
    (x,y)\in[0,\infty)^2
    \st
    \exists\gamma\,\lambda(\gamma,0)=x\text{ and }\lambda(\gamma,1)=y
  \right\}.
\end{equation*}
In many cases,
the prediction set is the boundary of the superprediction set.
The prediction set can also be defined as the set of points
\begin{equation}\label{eq:Lambda}
  \Lambda_{\gamma}
  :=
  \left(
    \lambda(\gamma,0),
    \lambda(\gamma,1)
  \right)
\end{equation}
where $\gamma$ ranges over the prediction space $[0,1]$.
It is clear that the prediction set is compact.

Let us fix a constant $\eta>0$.
The prediction set of the generalized log loss game
(\ref{eq:generalized-log})
is the  curve $\{(x,y)\st \e^{-\eta x}+\e^{-\eta y}=1\}$ in $\bbbr^2$.
For each $\pi\in(0,1)$,
the \emph{$\pi$-point} of this curve is $\Lambda_{\pi}$,
i.e., the point
$$
  \left(
    -\frac{1}{\eta} \ln(1-\pi),
    -\frac{1}{\eta} \ln\pi
  \right).
$$
Since the generalized log loss function is proper,
the minimum of $(1-\pi)x+\pi y$ on the curve $\e^{-\eta x}+\e^{-\eta y}=1$
is attained at the $\pi$-point;
in other words,
the tangent of $\e^{-\eta x}+\e^{-\eta y}=1$
at the $\pi$-point is orthogonal to the vector $(1-\pi,\pi)$.

A \emph{shift} of the curve $\e^{-\eta x}+\e^{-\eta y}=1$
is the curve $\e^{-\eta(x-\alpha)}+\e^{-\eta(y-\beta)}=1$
for some $\alpha,\beta\in\bbbr$
(i.e., it is a parallel translation of $\e^{-\eta x}+\e^{-\eta y}=1$
by some vector $(\alpha,\beta)$).
The \emph{$\pi$-point} of this shift is the point $(\alpha,\beta)+\Lambda_{\pi}$,
where $\Lambda_{\pi}$ is the $\pi$-point of the original curve
$\e^{-\eta x}+\e^{-\eta y}=1$.
This provides us with a coordinate system
on each shift of $\e^{-\eta x}+\e^{-\eta y}=1$
($\pi\in(0,1)$ serves as the coordinate of the corresponding $\pi$-point).

It will be convenient to use the geographical expressions
``Northeast'' and ``Southwest''.
A point $(x_1,y_1)$ is \emph{Northeast} of a point $(x_2,y_2)$
if $x_1\ge x_2$ and $y_1\ge y_2$.
A set $A\subseteq\bbbr^2$ is \emph{Northeast}
of a shift of $\e^{-\eta x}+\e^{-\eta y}=1$
if each point of $A$ is Northeast of some point of the shift.
Similarly, a point is \emph{Northeast} of a shift of $\e^{-\eta x}+\e^{-\eta y}=1$
(or of a straight line with a negative slope)
if it is Northeast of some point on that shift (or line).
``Northeast'' is replaced by ``Southwest''
when the inequalities are $\le$ rather than $\ge$,
and we add the attribute ``strictly'' when the inequalities are strict.

It is easy to see that the loss function is $\eta$-mixable
if and only if for each point $(a,b)$ on the boundary of the superprediction set
there exists a shift of $\e^{-\eta x}+\e^{-\eta y}=1$
passing through $(a,b)$
such that the superprediction set lies
to the Northeast of the shift.
This follows from the fact that the shifts of $\e^{-\eta x}+\e^{-\eta y}=1$
correspond to the straight lines with negative slope
under the homeomorphism $E_{\eta}$:
indeed, the preimage of $ax+by=c$, where $a>0$, $b>0$, and $c>0$,
is $a\e^{-\eta x}+b\e^{-\eta y}=c$,
which is the shift of $\e^{-\eta x}+\e^{-\eta y}=1$
by the vector
$$
  \left(
    -\frac{1}{\eta} \ln\frac{a}{c},
    -\frac{1}{\eta} \ln\frac{b}{c}
  \right).
$$
A similar statement for the property of being proper is:
\begin{lemma}\label{lem:proper}
  Suppose the loss function $\lambda$ is $\eta$-mixable.
  It is a proper loss function if and only if
  for each $\pi$
  the superprediction set is to the Northeast of the shift
  of $\e^{-\eta x}+\e^{-\eta y}=1$
  passing through $\Lambda_{\pi}$
  (as defined by (\ref{eq:Lambda}))
  and having $\Lambda_{\pi}$ as its $\pi$-point.
\end{lemma}
\begin{proof}
  The part ``if'' is obvious,
  so we will only prove the part ``only if''.
  Let $\lambda$ be $\eta$-mixable and proper.
  Suppose there exists $\pi$
  such the shift $A_1$ of $\e^{-\eta x}+\e^{-\eta y}=1$
  passing through $\Lambda_{\pi}$
  and having $\Lambda_{\pi}$ as its $\pi$-point
  has some superpredictions strictly to its Southwest.
  Let $s$ be such a superprediction,
  let $A_2$ be the shift of $\e^{-\eta x}+\e^{-\eta y}=1$
  passing through $\Lambda_{\pi}$ and $s$,
  and let $A_3$ be the tangent to $A_1$ at the point $\Lambda_{\pi}$.
  Then there are points on $A_2$ between $\Lambda_{\pi}$ and $s$
  that lie strictly to the Southwest of $A_3$
  (take any point on $A_2$ between $\Lambda_{\pi}$ and $s$
  that is sufficiently close to $\Lambda_{\pi}$).
  By the $\eta$-mixability of $\lambda$ these points
  must be superpredictions,
  which contradicts $\lambda$ being a proper loss function
  (since $A_3$ is the straight line passing through $\Lambda_{\pi}$
  and orthogonal to $(1-\pi,\pi)$).
\end{proof}

Notice that we never assume our loss functions
to be strictly proper.
(Geometrically,
the difference between proper mixable loss functions
and strictly proper mixable loss functions
is that the former's prediction set is allowed to have corners.)

\subsection*{Proof of the supermartingale property}

Let $E\subseteq([0,1]^N\times(0,\infty)^N\times\LLL^N)$
consist of sequences
$$
  \left(
    \gamma^1,\ldots,\gamma^N,
    \eta^1,\ldots,\eta^N,
    \lambda^1,\ldots,\lambda^N
  \right)
$$
such that $\gamma^n$ is $\eta^n$-mixable for all $n=1,\ldots,N$.
We will only be interested in the restriction of $Q^n$ and $Q$
on $(E\times[0,1]\times\{0,1\})^*$;
these restrictions are denoted with the same symbols.

The following lemma completes the proof of Theorem \ref{thm:main}.
We will prove it without calculations,
unlike the proofs
(of different but somewhat similar properties)
presented in \cite{chernov/etal:2008supermartingales}
(and, specifically for the binary case, in \cite{vovk:arXiv0708.1503}).
\begin{lemma}\label{lem:supermartingale-general}
  The function $Q^n$ defined
  on $(E\times[0,1]\times\{0,1\})^*$ by (\ref{eq:Q-n})
  is a supermartingale.
\end{lemma}
\begin{proof}
  It suffices to check that it is always true that
  \begin{multline*}
    \pi_T
    \exp
    \left(
      \eta^n_T
      \left(
        \lambda^n_T(\pi_T,1)
        -
        \lambda^n_T(\gamma^n_T,1)
      \right)
    \right)\\
    +
    (1-\pi_T)
    \exp
    \left(
      \eta^n_T
      \left(
        \lambda^n_T(\pi_T,0)
        -
        \lambda^n_T(\gamma_T^n,0)
      \right)
    \right)
    \le
    1.
  \end{multline*}
  To simplify the notation, we omit the indices $n$ and $T$;
  this does not lead to any ambiguity.
  Using the notation
  $(a,b):=\Lambda_{\pi}=(\lambda(\pi,0),\lambda(\pi,1))$
  and
  $(x,y):=\Lambda_{\gamma}=(\lambda(\gamma,0),\lambda(\gamma,1))$,
  we can further simplify the last inequality to
  \begin{equation*}
    (1-\pi)
    \exp
    \left(
      \eta
      \left(
        a - x
      \right)
    \right)
    +
    \pi
    \exp
    \left(
      \eta
      \left(
        b - y
      \right)
    \right)
    \le
    1.
  \end{equation*}
  In other words,
  it suffices to check that the (super)prediction set
  lies to the Northeast of the shift
  \begin{equation}\label{eq:star}
    \exp
    \left(
      -\eta
      \left(
        x
        -
        a
        -
        \frac{1}{\eta}\ln(1-\pi)
      \right)
    \right)
    +
    \exp
    \left(
      -\eta
      \left(
        y
        -
        b
        -
        \frac{1}{\eta}\ln\pi
      \right)
    \right)
    =
    1
  \end{equation}
  of the curve $\e^{-\eta x}+\e^{-\eta y}=1$.
  The vector by which (\ref{eq:star}) is shifted is
  $$
    \left(
      a
      +
      \frac{1}{\eta}\ln(1-\pi),
      b
      +
      \frac{1}{\eta}\ln\pi
    \right),
  $$
  and so $(a,b)$ is the $\pi$-point of that shift.
  This completes the proof of the lemma:
  by Lemma \ref{lem:proper},
  the superprediction set indeed lies to the Northeast of that shift.
\end{proof}

\subsection*{A simple special case}

In the case where $\lambda^n_t=\lambda$ is the log loss function
and $\eta^n_t=1$ for all $n$ and $t$,
the supermartingale (\ref{eq:Q-n})
(which is in fact a martingale now)
becomes a likelihood ratio process:
namely, it becomes the ratio
$$
  \prod_{t=1}^T
  \frac{\tilde{\gamma}^n_t(\{\omega_t\})}{\tilde{\pi}_t(\{\omega_t\})},
$$
where $\tilde p$, $p\in[0,1]$,
stands for the probability measure on $\{0,1\}$
such that $\tilde{p}(\{1\})=p$.
The mixed martingale $Q$ becomes the likelihood ratio
with the Bayes mixture as the numerator,
and it is easy to see that in this case
defensive forecasting reduces to the Bayes rule.

\section{Defensive forecasting for specialist experts and the AA}

In this section we will find a more explicit version
of defensive forecasting in the case of specialist experts.
Our algorithm will achieve a slightly more general version
of the bound (\ref{eq:sleeping});
namely,
we will replace the $\ln N$ in (\ref{eq:sleeping})
by $-\ln p^n$ where $p^n$ is an \emph{a priori} chosen weight
for Expert $n$:
all $p^n$ are non-negative and sum to $1$.
Without loss of generality all $p^n$ will be assumed positive
(our algorithm can always be applied to the subset of experts
with positive weights).
Let $A_t$ be the set of awake experts at time $t$:
$A_t:=\{n\in\{1,\ldots,N\}\st\gamma_t^n\ne a\}$.

Let $\lambda$ be an $\eta$-mixable loss function.
By the definition of mixability
there exists a function $\Sigma(u_1,\ldots,u_k,\gamma_1,\ldots,\gamma_k)$
(called a \emph{substitution function})
such that:
\begin{itemize}
\item
  the domain of $\Sigma$
  consists of all sequences $(u_1,\ldots,u_k,\gamma_1,\ldots,\gamma_k)$,
  for all $k=0,1,2,\ldots$,
  of numbers $u_i\in[0,1]$ summing to 1, $u_1+\cdots+u_k=1$,
  and predictions $\gamma_1,\ldots,\gamma_k\in[0,1]$;
\item
  $\Sigma$ takes values in the prediction space $[0,1]$;
\item
  for any $(u_1,\ldots,u_k,\gamma_1,\ldots,\gamma_k)$ in the domain of $\Sigma$,
  the prediction $\gamma:=\Sigma(u_1,\ldots,u_k,\gamma_1,\ldots,\gamma_k)$ satisfies
  \begin{equation}\label{eq:substitution}
    \forall\omega\in\{0,1\}:
    \e^{-\eta\lambda(\gamma,\omega)}
    \ge
    \sum_{i=1}^k
    \e^{-\eta\lambda(\gamma_i,\omega)}
    u_i.
  \end{equation}
\end{itemize}
Fix such a function $\Sigma$.
Notice that its value $\Sigma()$ on the empty sequence
can be chosen arbitrarily,
that the case $k=1$ is trivial,
and that the case $k=2$ in fact covers the cases $k=3$, $k=4$, etc.

\bigskip

\noindent
\mbox{\textsc{Defensive forecasting algorithm
  for specialist experts}}\nopagebreak

\smallskip

\parshape=9
\IndentI   \WidthI
\IndentI   \WidthI
\IndentII  \WidthII
\IndentIII \WidthIII
\IndentII  \WidthII
\IndentIII \WidthIII
\IndentII  \WidthII
\IndentII  \WidthII
\IndentI   \WidthI
\noindent
$w^n_0:=p^n$, $n=1,\ldots,N$.\\
FOR $t=1,2,\dots$:\\
  Read the list $A_t$ of awake experts\\
    and their predictions $\gamma_t^n\in[0,1]$, $n\in A_t$.\\
  Predict
    $
      \pi_t
      := 
      \Sigma
      \left(
        \left(u^n_{t-1}\right)_{n\in A_t},
        \left(\gamma^n_t\right)_{n\in A_t}
      \right)
    $,\\
    where $u^n_{t-1} := w^n_{t-1} / \sum_{n\in A_t} w^n_{t-1}$.\\
  Read the outcome $\omega_t\in\{0,1\}$.\\
  Set
    $
      w^n_t
      :=
      w^n_{t-1}
      \e^{\eta(\lambda(\pi_t,\omega_t)-\lambda(\gamma^n_t,\omega_t))}
    $
    for all $n\in A_t$.\\
END FOR

\bigskip

\noindent
This algorithm is a simple modification of the AA,
and it becomes the AA when the experts are always awake.
In the case of the log loss function,
this algorithm was found by Freund et al.\ \cite{freund/etal:1997};
in this special case,
Freund et al.\ derive the same performance guarantee as we do.

\subsection*{Derivation of the algorithm}

In this derivation we will need the following notation.
For each history of the game,
let $A^n$, $n\in\{1,\ldots,N\}$,
be the set of steps at which Expert $n$ is awake:
$$
  A^n
  :=
  \{t\in\{1,2,\ldots\}\st n\in A_t\}.
$$
For each positive integer $k$,
$[k]$ stands for the set $\{1,\ldots,k\}$.

The method of defensive forecasting requires
(cf.\ Corollary \ref{cor:sleeping})
that at step $T$ we should choose $\pi=\pi_T$
such that, for each $\omega\in\{0,1\}$,
\begin{multline*}
  \sum_{n\in A_T}
  p^n
  \e^{\eta(\lambda(\pi,\omega)-\lambda(\gamma^n_T,\omega))}
  \prod_{t\in[T-1]\cap A^n}
  \e^{\eta(\lambda(\pi_t,\omega_t)-\lambda(\gamma^n_t,\omega_t))}\\
  +
  \sum_{n\in A_T^c}
  p^n
  \prod_{t\in[T-1]\cap A^n}
  \e^{\eta(\lambda(\pi_t,\omega_t)-\lambda(\gamma^n_t,\omega_t))}\\
  \le
  \sum_{n\in[N]}
  p^n
  \prod_{t\in[T-1]\cap A^n}
  \e^{\eta(\lambda(\pi_t,\omega_t)-\lambda(\gamma^n_t,\omega_t))}
\end{multline*}
where $A_T^c$ stands for the complement of $A_T$ in $[N]$:
$A_T:=[N]\setminus A_T$.
This inequality is equivalent to
\begin{multline*}
  \sum_{n\in A_T}
  p^n
  \e^{\eta(\lambda(\pi,\omega)-\lambda(\gamma^n_T,\omega))}
  \prod_{t\in[T-1]\cap A^n}
  \e^{\eta(\lambda(\pi_t,\omega_t)-\lambda(\gamma^n_t,\omega_t))}\\
  \le
  \sum_{n\in A_T}
  p^n
  \prod_{t\in[T-1]\cap A^n}
  \e^{\eta(\lambda(\pi_t,\omega_t)-\lambda(\gamma^n_t,\omega_t))}
\end{multline*}
and can be rewritten as
\begin{equation}\label{eq:requirement}
  \sum_{n\in A_T}
  \e^{\eta(\lambda(\pi,\omega)-\lambda(\gamma^n_T,\omega))}
  u^n_{T-1}
  \le
  1,
\end{equation}
where $u^n_{T-1}:=w^n_{T-1}/\sum_{n\in A_T}w^n_{T-1}$
are the normalized weights
$$
  w^n_{T-1}
  :=
  p^n
  \prod_{t\in[T-1]\cap A^n}
  \e^{\eta(\lambda(\pi_t,\omega_t)-\lambda(\gamma^n_t,\omega_t))}.
$$
Comparing (\ref{eq:requirement}) and (\ref{eq:substitution}),
we can see that it suffices to set
$$
  \pi
  :=
  \Sigma
  \left(
    \left(u^n_{T-1}\right)_{n\in A_T},
    \left(\gamma^n_T\right)_{n\in A_T}
  \right).
$$

\subsection*{Discussion of the algorithm}

The main difference of the algorithm of the previous subsection
from the AA
is in the way the experts' weights are updated.
The weights of the sleeping experts are not changed,
whereas the weights of the awake experts are multiplied by
$\e^{\eta(\lambda(\pi_t,\omega_t)-\lambda(\gamma^n_t,\omega_t))}$.
Therefore, Learner's loss serves as the benchmark:
the weight of an awake expert who performs better than Learner goes up,
the weight of an awake expert who performs worse than Learner goes down,
and the weight of a sleeping expert does not change.

\subsection*{Acknowledgements}

We are grateful to the anonymous Eurocrat
who coined the term ``expert evaluator''.
This work was supported in part by EPSRC grant EP/F002998/1.

\end{document}